\documentclass{article}

 \usepackage[final]{nips_2016}

\usepackage[utf8]{inputenc} 
\usepackage[T1]{fontenc}    
\usepackage{url}            
\usepackage{natbib}
\usepackage{booktabs}       
\usepackage{array}
\usepackage{amsfonts}       
\usepackage{nicefrac}       
\usepackage{microtype}      
\usepackage{times,amsmath, amsthm, amssymb,epsfig}
\usepackage{graphicx,subfigure}
\usepackage{tablefootnote}
\usepackage[linesnumbered,ruled,vlined,noend]{algorithm2e}
\usepackage{balance}  
\usepackage{bbm}
\usepackage{epstopdf} 
\usepackage{color}
\usepackage{floatrow}
\newfloatcommand{capbtabbox}{table}[][\FBwidth]
\newtheorem{prop}{Proposition}
\usepackage[
singlelinecheck=false
]{caption}
\setcitestyle{authoryear}

\DeclareMathOperator{\Var}{Var}
\DeclareMathOperator{\Uni}{Unif}
\DeclareMathOperator{\iid}{iid}
\title{Variational hybridization and transformation for \\large inaccurate noisy-or networks}

\author{
    Yusheng Xie\hspace{1cm}Nan Du\hspace{1cm}Wei Fan\hspace{1cm}Jing Zhai\hspace{1cm}Weicheng Zhu\\
  Baidu Research\\
  Sunnyvale, CA 94089 \\
  \texttt{\{xieyusheng,dunan,fanwei03,zhaijing01,zhuweicheng\}@baidu.com} \\
}

\begin{document}

\maketitle

\begin{abstract}
Variational inference provides approximations to the computationally intractable posterior distribution in Bayesian networks. A prominent medical application of noisy-or Bayesian network is to infer potential diseases given observed symptoms. Previous studies focus on approximating a handful of complicated pathological cases using variational transformation. Our goal is to use variational transformation as part of a novel hybridized inference for serving reliable and real time diagnosis at web scale. We propose a hybridized inference that allows variational parameters to be estimated without disease posteriors or priors, making the inference faster and much of its computation recyclable. In addition, we propose a transformation ranking algorithm that is very stable to large variances in network prior probabilities, a common issue that arises in medical applications of Bayesian networks. In experiments, we perform comparative study on a large real life medical network and scalability study on a much larger (36,000x) synthesized network.
\end{abstract}

\section{Introduction}

Noisy-or Bayesian network (NOBN) is a popular class of statistical models in modeling observable events and their unobserved potential causes. 
One of the best known medical applications of NOBN is Quick Medical Reference (\texttt{QMR-DT})~\citep{qmr}.
\texttt{QMR-DT} describes expert-assessed relationships between 4,000+ observable binary symptom variables (collectively denoted as $S$) and 500+ binary latent disease variables (collectively denoted as $D$) as illustrated in Figure \ref{fig:intro} (a).

We improve variational inference for a large \texttt{QMR-DT} style NOBN in areas of scalability, stability, and accuracy to previously unattainable or untested levels. As part of a medical messaging bot, the inference goal is to perform reliable real time diagnosis at web scale. Figure \ref{fig:intro} (b) shows the messaging bot's interface. The ongoing project aims to serve a substantial portion of Internet users who experience health issues (e.g., 3 to 8 million daily active users\footnote{Assume an average person is sick 2-4 days per year and our reachable population is 600 to 800 million.}) with reliable disease diagnosis that is more accurate and accessible than text-based web searches, web searches that emphasize retrieval similarity but lack clinical technicality (e.g., \emph{38.5 $^{\circ}$C fever lasting 3 days} and \emph{39.5 $^{\circ}$C fever lasting 8 days}. The latter could be 20x more fatal in probability). The developing bot has completed 1,000+ organic, non-scripted dialogues with 100+ qualified human testers. Assessed by 50+ licensed doctors, the network plans to cover \emph{all} conceivable human diseases and health conditions\footnote{A sub-network focusing on maternal and infant care is completed and used in our experiments.}: approximately 40,000 (80x that of \texttt{QMR-DT}) according to \emph{the 10th revision of the International Statistical Classification of Diseases and Related Health Problems} (ICD-10). To the best of our knowledge, the aforementioned scales make it the largest medical application of noisy-or Bayesian networks. 

Recent advances in modern machine learning and artificial intelligence quickly proliferate far beyond the traditional Bayesian framework. But for mission critical applications such as medical diagnosis, one prefers Bayesian network-based approach for reasoning instead of entirely data driven approach. The reasons are due to traceable outcome, easy debuggability, and  provenance.
Data source unreliability and scarcity also prevent some medical applications from taking full advantage of the large body of data driven algorithms that can be quickly accelerated by larger datasets (e.g., machine translation~\citep{nmt}, speech recognition~\citep{deepspeech}). For example, the \emph{Caroli} disease\footnote{Caroli disease is a type of congenital dilatation of intrahepatic bile duct. It has the code Q44.6 in ICD-10.} has fewer than 250 recorded cases worldwide, making it almost impossible to ``gather/label more data points''. On the other hand, no disease should be too rare to deserve attention From an ethical perspective, even a 1-in-1,000,000 chance (technically \emph{extremely rare}) translates into over 6,000 suffering individuals worldwide. From an academic perspective, understanding rare diseases brings irreplaceable medical knowledge.

 The expert-assessed probability of observing symptom $f$ given \emph{only} disease $d$ is denoted as $P(f^+\mid d^+)$. We use $\pi(f)$ to denote $\{d\mid d\in D, P(f^+\mid d^+) >0\}$, the set of diseases that could cause $f$ with non-zero probability. Like \texttt{QMR-DT}, we assume\footnote{Without loss of generality, the leak probabilities~\citep{j99} are omitted in our discussion.} $P(d^+)$ for each $d\in D$: the prior probability of having disease $d$ without observing any symptoms. We further define $P(f^-)$ and $P(d^-)$ notations as $P(f^-) = 1 - P(f^+)$ and $P(d^-) = 1 - P(d^+)$, respectively.

In a typical diagnosis session, the user first inputs her positive and negative findings: $F^+=\{f^+_1,f^+_2,\ldots\}\subset S$, $F^-=\{f^-_1,f^-_2,\ldots\}\subset S$. Then the model performs inference to calculate $P(F^+,F^-)$, which is the crux in deriving the conditional $P(d_i^+ \mid F^+,F^-)$ for each $d\in D$. 

\begin{figure}
\centering
\subfigure[]{
   \includegraphics[width=4.4cm]{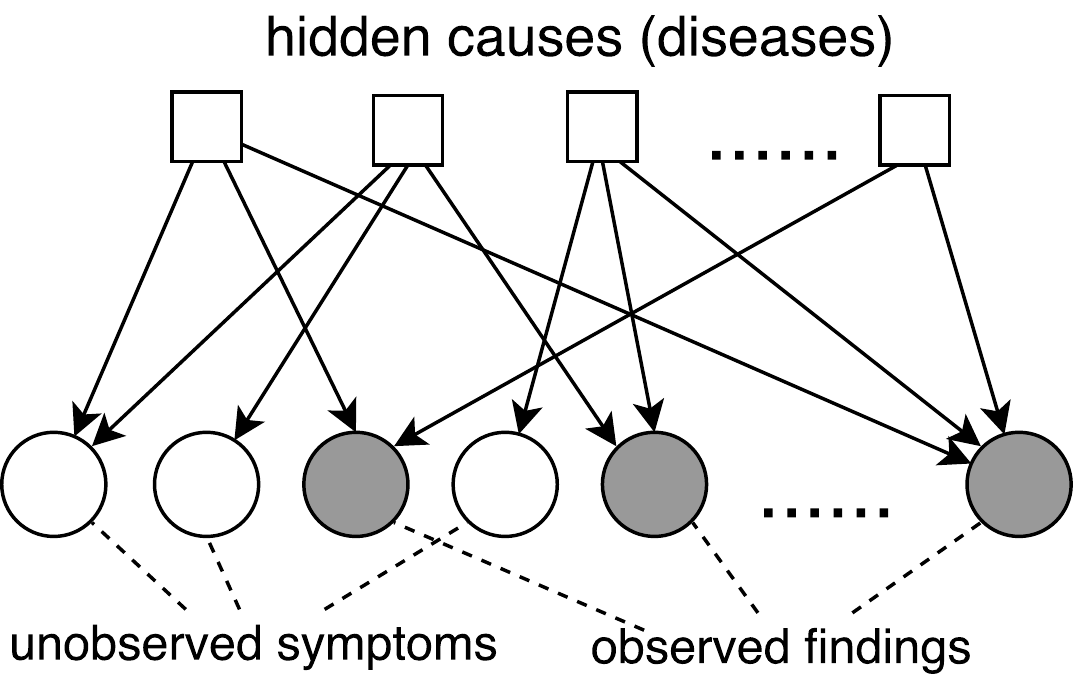}
}\hspace{-0.5mm}%
\subfigure[]{
  \includegraphics[width=4.4cm]{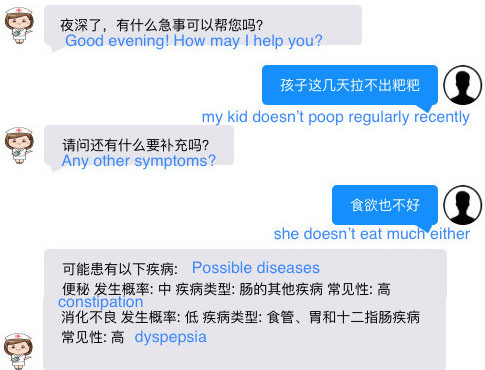}
}\hspace{-0.5mm}%
\subfigure[]{
   \includegraphics[width=4.4cm]{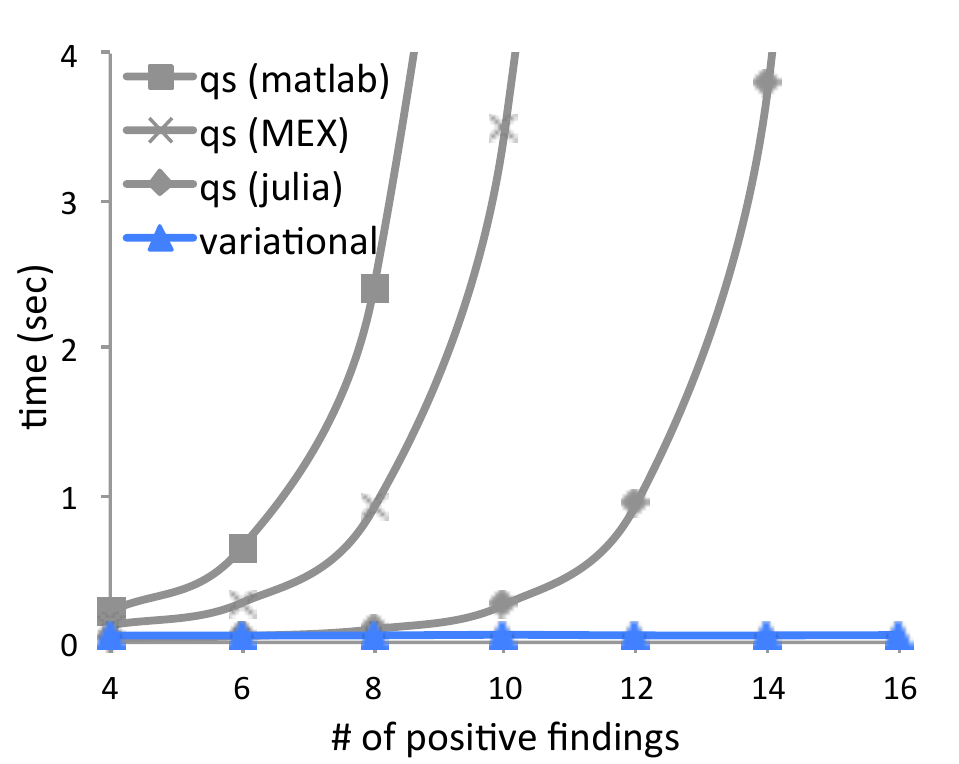}
}
\vspace{-0.5cm}
\caption{(a) Graphical model structure of \texttt{QMR-DT}. The shaded round nodes are observed nodes ($f^+$ or $f^-$). All variables are binary. (b) Screenshot of the diagnosis bot. (c) Running time comparison of exact Quickscore (qs) and variational inference for $|F^+|=4,6,\ldots,16$. \texttt{qs(matlab)} and \texttt{qs(MEX)} are provided by \citep{bnt}. \texttt{qs(julia)} and \texttt{variational} are authors' own implementation.}
\label{fig:intro}
\end{figure}

\subsection{Background on variational inference}
The exact inference for $P(F^+,F^-)$ is intractable~\citep{cooper90} and intractability motivates investigations into approximation inference algorithms. The variational method~\citep{j99,jj99} and the mean field local approximation~\citep{ng00} are both hybrid approximation algorithms.

To describe the variational approximation, let $\theta_{ji} \equiv -\log P(f_j^-\mid d^+_i)$. \citep{j99,jj99} show that
\small
\begin{equation}\label{e4}
\begin{split}
P\left(f_j^+\mid \pi(f_j)^+\right) = e^{f\left( \sum_{i=1}^{|\pi(f_j)|}\theta_{ji} \right)} \le e^{\sum_{i=1}^{|\pi(f_j)|}\xi_j\theta_{ji} -f^*(\xi_j)}\equiv P(f_j^+\mid \pi(f_j)^+,\xi_j)\\
\end{split}
\end{equation}
\normalsize
and
\small
\begin{equation}\label{e5}
\begin{split}
P(f_j^+\mid \xi_j) &=\prod_{d_i\in \pi(f_j)}\left[ P(f_j^+\mid d_i^+,\xi_j)\cdot P(d_i^+) + P(f_j^+\mid d_i^-,\xi_j)\cdot P(d_i^-)\right]\ge P(f_j^+),\\
\end{split}
\end{equation}
\normalsize
where $\xi_j$ is the free variational parameter, $f(x)\equiv\log\left(1-e^{-x}\right)$, and $f(x)$'s convex conjugate function takes the form $f^*(\xi)=-\xi \log \xi + (\xi+1) \log(\xi+1)\text{, for }\xi>0$.
Equation \ref{e5} transforms $P(f_j^+)$ into its variational upper bound $P(f_j^+\mid \xi_j)$ using the inequality from conjugate duality.

Breaking $F^+=\{f^+_1,f^+_2,\ldots\}$ into the partition $F_1^+$ and $F_2^+$ allows exact inference on $F_2^+$ and variational inference on $F_1^+$. \citep{jj99} (JJ99) calculates the joint variational posterior as
\small
\begin{equation}\label{e10}
\begin{split}
P&_{JJ99}(F_1^+, F_2^+,F^-\mid\Xi^{\min}) \\
& = e^{- \sum_{j=1}^{|F_1^+|} f^*(\xi^{\min}_{j})} \prod_{d_i\in \pi(F_1^+)} \left[e^{\sum_{j=1}^{|F^+|}\xi^{\min}_{j} \theta_{ji}} \cdot P(d^+_i\mid F_2^+,F^-)+ P(d^-_i\mid F_2^+,F^-) \right],
\end{split}
\end{equation}
\normalsize
where $\Xi=\{\xi_1,\xi_2,\ldots\}$.
Finding $\arg\min_{\Xi} P(F^+\mid \Xi)$ can be relaxed to finding $\arg\min_{\xi_j} \log P(f_j^+\mid \xi_j)$ for each $\xi_j\in\Xi$. 
The $\xi$-convexity permits second order optimization methods (CVX) to find each $\xi_j$. From Equation \ref{e7}, the 1st order partial derivatives are 
\small
\begin{equation}\label{e9}
\begin{split}
{\partial\over \partial \xi_j}\log P(f_j^+\mid \xi_j)  = \log{\xi_j \over 1+\xi_j} + \sum_{d_i\in \pi(f_j^+)} {\theta_{ji}\over p_i\cdot e^{-\xi_j \theta_{ji}}+1},\\
\end{split}
\end{equation}
\normalsize
where $p_i\equiv P(d^-_i)/P(d^+_i)$ is the inverse prior odds for the $i$th disease. The 2nd order partial derivatives are derived mechanically.
Figure \ref{fig:intro} (c) illustrates the complexity of exact and variational inference in real application.
More discussion on existing inference algorithms are in related works section (see Table \ref{tab:nobn}).

\section{Inaccuracy in widely-ranged disease priors}
Inaccurate hidden variable prior is a recognized~\citep{nips13,uai06} but often avoided~\citep{ai02,pr09,ijar06b} issue in NOBN.
Inaccuracy in disease prior is among the most likely errors in constructing a NOBN for medical applications. Real life disease priors can span several orders of magnitude. For example, \emph{acne} (ICD-10 code: L70.0) affects 80\% to 90\% teenagers in the western world~\citep{acne} while syndromes like the \emph{Caroli} disease have historical infection rates less than 0.00001\%. It is very likely, even for medical experts or statistical estimators, to misjudge the prior probability by an order of magnitude relative to other very rare or very common diseases. So it is beneficial to obtain fast and accurate variational algorithms that are \emph{resistant} to the large variances in disease priors.

In the following two sections, we propose inference algorithms that can greatly immunize the current variational inference against inaccuracy in disease priors.

\section{Variational-first hybridization and joint hybridization}
The $F^+=F^+_1\cup F^+_2$ partition employed in JJ99 is a realization of the classic hybrid paradigm: balancing accuracy and runtime over the entire $F^+$ by 1) applying different posterior estimators (variational, exact, MCMC, etc.) to $F_1^+, F_2^+$, and 2) controlling their cardinalities.
But JJ99 has two main drawbacks that prevent it from fulfilling the scalability and stability requirements in building a web diagnostic bot.

First, Equation \ref{e10} estimates $\Xi^{\min}$ by using the exactly treated disease posterior $P(d^+_i\mid F_2^+,F^-)$. The $\Xi^{\min}$ estimations need be recalculated for every case of $F_1^+\cup F_2^+\cup F^-$ since each case would produce different disease posteriors that affect the gradients in Equation \ref{e9}. 
Second, in order to pass confident posteriors to its variational step, JJ99 basically ``primes'' the potentially inaccurate disease priors with evidences from $F_2^+\cup F^-$. Since the hybridized complexity decreases exponentially w.r.t. to $|F^+_1|$, $F^+_1$ usually contains less evidence than $F_2^+\cup F^-$ in practice (i.e., $|F^+_1|<|F_2^+\cup F^-|$). In other words, JJ99 uses a substantial portion of the evidence in priming the \emph{unaudited} priors first and then refines the posterior probabilities using the smaller leftover portion of evidence.

We propose the variational-first hybridization (VFH) that can fix both issues. Described in Algorithm \ref{algo:vfh}, VFH performs inference on $F_1^+$ first (to prime the unaudited priors) and on $F_2^+$ and $F^-$ later (to refine the posteriors). Calculating $\Xi^{\min}$ in VFH relies on disease priors instead of posteriors. Therefore, the calculation is invariant to the findings that make up $F_1^+\cup F_2^+\cup F^-$. Invariant $\Xi^{\min}$ allows caching $\Xi^{\min}$ values and leads to faster inference as summarized in Table \ref{tab:variational}.

Equation \ref{e10.5} explicitly expresses the joint variational evidence of given findings using VFH:
\small
\begin{equation}\label{e10.5}
\begin{split}
P&_{VFH}(F_1^+, F_2^+,F^-\mid\Xi^{\min}) \\
& =\sum_{F'\in 2^{F_2^+}}\left( -1\right)^{|F'|}\prod_{i=1}^{|D|}\left(\left[\prod_{j=1}^{|F^-\cup F'|}P(f^-_j\mid d^+_i)\right] P(d^+_i\mid F_1^+, \Xi^{\min})+ P(d^-_i\mid F_1^+, \Xi^{\min})\right),\\
\end{split}
\end{equation}
\normalsize
where $2^{F_2^+}$ denotes the power set of $F_2^+$ and the $P(d^+_i\mid F_1^+, \Xi^{\min})$ terms are calculated from
\small
\begin{equation}\label{e7}
\begin{split}
P(F^+\mid \Xi) =e^{- \sum_{j=1}^{|F^+|} f^*(\xi_j)} \prod_{d_i\in \pi(F^+)} \left[e^{\sum_{j=1}^{|F^+|}\xi_j \theta_{ji}} \cdot P(d^+_i)+ P(d^-_i) \right].
\end{split}
\end{equation}
\normalsize

Besides VFH, we can also hybridize the exact evidence $P(F_2^+,F^-)$ and the variational evidence $P(F_1^+\mid \Xi)$ jointly  (JH):
\vspace{-0.1cm}
\tiny
\begin{equation}\label{e11}
\begin{split}
P&_{JH}\left(F_1^+,F_2^+,F^-\mid \Xi^{\min}\right) \\
& =\hspace{-6pt}\sum_{F'\in 2^{F_2^+}}\hspace{-6pt}\left( -1\right)^{|F'|}\prod_{i=1}^{|D|}\hspace{-2pt}\left(\left[\prod_{j=1}^{|F^-\cup F'|}\hspace{-2pt}P(f^-_j\hspace{-2pt}\mid\hspace{-2pt} d^+_i)\right]\hspace{-4pt}\left[\prod_{k=1}^{|F_1^+|}\hspace{-2pt}P(f^+_k\hspace{-2pt}\mid\hspace{-2pt} d^+_k,\xi_{k}^{\min})\right] \hspace{-3pt} P(d^+_i) + \left[\prod_{k=1}^{|F_1^+|}P(f^+_k\mid d^-_k,\xi_{k}^{\min})\right]\hspace{-3pt}P(d^-_i)\hspace{-2pt}\right)\\
& =\left[e^{- \sum_{k=1}^{|F_1^+|} f^*(\xi^{\min}_{k})}\right]\hspace{-3pt}\sum_{F'\in 2^{F_2^+}}\hspace{-6pt}\left( -1\right)^{|F'|}\prod_{i=1}^{|D|}\left(\left[\prod_{j=1}^{|F^-\cup F'|}P(f^-_j\mid d^+_i)\right]\left[ e^{\sum_{k=1}^{|F_1^+|}\xi^{\min}_{k} \theta_{ki}} \right] P(d^+_i)+ P(d^-_i)\right).\\
\end{split}
\end{equation}
\normalsize

Like VFH, JH has the same advantages over JJ99 when $|F^+_1|<|F_2^+\cup F^-|$ .

\begin{algorithm}[t!]
    \SetAlgoLined
    \DontPrintSemicolon 
    \KwIn{$F_1^+$, list of positive findings to be inferred variationally, $F_2^+$, list of positive findings to be inferred exactly, $F^-$, list of negative findings to be inferred exactly, $\theta_{ji}$ for each $f_j\in S$ and $d_i\in D$, $P(d_i)$, disease prior probability for each $d_i\in D$. }
    \KwOut{The joint variational evidence of given findings $F_1^+, F_2^+, $ and $F^-$.}
    Calculate $P(F_1^+\mid \Xi)$ as a function of $\Xi$ from Equation \ref{e7}.\;
    $\Xi^{\min} \leftarrow \arg\min_{\Xi} P(F_1^+\mid \Xi)$ using Newton's method on its derivatives (shown in Equation \ref{e9}).\;
    \For{each $d_i\in D$}{
    Calculate $P(d_i^+ \mid F_1^+, \Xi^{\min})$ from Equation \ref{e7} and $P(F_1^+\mid \Xi)$.\;
    $P(d_i) \gets P(d_i^+\mid F^+_1, \Xi^{\min})$ (update disease priors with posteriors).\;
    }
    $P(F^+_2,F^-)\gets\text{Quickscore}(F^+_2,F^-)$.\;    
    \Return{$P( F^+_2,F^-)$}\;
\caption{the proposed variational-first hybridization (VFH) algorithm.}
\label{algo:vfh}
\end{algorithm}

\subsection{Estimate $\xi^{\min}_{j}$ without disease prior or posterior} 

If we solve $\xi_j^{\min}$ from $\arg\min_{\xi_j} \log P\left(f_j^+\mid \xi_j , \pi(f_j)^+\right)$ instead of $\arg\min_{\xi_j} \log P\left(f_j^+\mid \xi_j\right)$, the resulting $\xi_j^{\min}$ has a closed form solution. To see this, take the equality in Equation \ref{e4}
and let $x_j\equiv\sum_{i=1}^{|\pi(f^+_j)|}\theta_{ji}$. The equality  $e^{f\left( x_j\right)}=e^{\xi_j x_j - f^*\left( \xi_j\right)}$ holds if and only if  $\xi^{\min}_j=\arg\min_{\xi_j}\xi_j x_j - f^*\left( \xi_j\right)$. Simple algebra gives the closed form $\xi^{\min}_j = \left( e^{x_j} - 1\right)^{-1}$. Conceptually, $\arg\min_{\xi_j} \log P\left(f_j^+\mid \xi_j , \pi(f_j)^+\right)$ would surely result in suboptimal $\xi^{\min}_{j}$ due to its lack of prior knowledge. However, we find this approach competitive for a certain range of disease priors (shown in experiments).
The prior/posterior-free (PPF) estimator of $\Xi^{\min}$  is independent of disease prior or posterior and allows $\xi^{\min}_j$ to be pre-computed and cached regardless of JJ99, VFH or JH.

\subsection{$N$-scalability of JJ99, VFH, and JH} 
The ability to process a large number ($N$) of diagnosis with low latency is quintessential for web scalability. The variational step in JJ99+CVX (baseline) is $O(N)$, which would put increasing strain on the server as $N$ grows.  On the other hand, the proposed VFH and JH perform the variational step in constant time w.r.t. $N$. With the proposed PPF estimator of $\Xi^{\min}$, all hybridization schemes can execute variational transformation in constant time w.r.t. $N$.
 Table \ref{tab:variational} summarizes the practical efficiency of the proposed variational hybridization when used with either CVX or PPF estimator of $\Xi^{\min}$. The $\log\log {1\over \epsilon}$ term is the optimization cost using second order algorithms like Newton's method. Note that Table \ref{tab:variational} only compares the cost of the variational step. We evaluate the overall inference cost for different inferencers in the Experiments section.

\begin{table}[ht!]
\centering
\caption{Detailed temporal complexities for the proposed variational parameter estimation in terms of $|D|$, $|F_1^+|$, $|F^-|$, $\epsilon$, $|S|$, and $N$. All entries are Big-$O$ complexity. Note that although JH is equivalent to VFH in variational parameter estimation, JH will have higher overall inference complexity due to difference between Equation \ref{e10.5} and \ref{e11}.}
\label{tab:variational}
\resizebox{13.5cm}{!} {\renewcommand{\arraystretch}{1.1}
\begin{tabular}{  c | c | c | c | c }

$\Xi^{\min}$ solver& \# of queries& JJ99 & VFH & JH \\
\hline
 & 1& \small  $|D|\cdot|F_1^+|\log\log {1\over \epsilon}$ & 
\small  $|D|\cdot|F_1^+|\log\log {1\over \epsilon}$ & 
\small  $|D|\cdot|F_1^+|\log\log {1\over \epsilon}$\\
 CVX& $N$& \small  $N\cdot|D|\cdot|F_1^+|\log\log {1\over \epsilon}$ & 
\small  $|D|\cdot|S|\log\log {1\over \epsilon}$ & 
 \small  $|D|\cdot|S|\log\log {1\over \epsilon}$\\
\hline 
 & 1& \small  $|D|\cdot|F_1^+|$ & 
\small  $|D|\cdot|F_1^+|$ & 
\small  $|D|\cdot|F_1^+|$\\
 PPF& $N$& \small  $N\cdot|D|\cdot|F_1^+|$ & 
\small  $|D|\cdot|S|$ & 
 \small  $|D|\cdot|S|$\\
\end{tabular}
}
\end{table}

\section{Variational transformation with uncertain disease priors}
In addition to the inference formula (JJ99, VFH, or JH) and the $\Xi^{\min}$ solver (CVX or PPF), there is a third component in variational inference that is critical to the posterior accuracy: the transformation ranking algorithm that partitions $F^+$ into $F^+_1$ and $F^+_2$, given fixed $|F^+_1|$.

\citep{jj99} and \citep{ng00} use a simple greedy heuristic ordering (GDO) algorithm to rank the order of transformation based on the greedy local optimum for further minimizing the overall variational upper bound (which is firstly minimized by setting $\Xi=\Xi^{\min}$). Minimizing the overall variational upper bound is, naturally, a commendable goal. But given the inaccuracy in widely-ranged disease priors, is there an ordering algorithm that can fender off that uncertainty more effectively than GDO?

To simplify the discussion, we assume uniform $\theta_{ji}=c$ for any $j,i$ pair such that $P(f^-_j\mid d^+_i)<1$, where $c\in(0,+\infty)$. 
Let the random variable (r.v.) $\mathcal{P}={1\over m }\sum_{k=1}^m \mathcal{U}_k$, where $\mathcal{U}_k \sim \iid \Uni(0,{2\over 1+p})$ for $k=1,2,\ldots,m$.  
We further assume that the inverse disease prior odds: $P\left(d_i^-\right)/P\left(d_i^+\right) = p_i$ for any $i\in\{1,\ldots,|D|\}$ are drawn independently from $\mathcal{P}$. 
The choice of $m$ is rather inconsequential in our discussion. For a reasonable $m$ (e.g., $5<m<1,000$), the uniform mean distribution $\mathcal{P}$ introduces Gaussian-like variance without breaking the positive definite constraint on $p_i$'s.

We desire to establish an ordering algorithm that minimizes the variance in posterior predictions due to $\mathcal{P}$. The first step is to show its existence.
Formally, it is stated and proved in Proposition \ref{p2}.

\begin{prop}\label{p2}
Fix $p\in[0,+\infty)$, $c\in(0,+\infty)$, and $n\in\{1,2,\ldots,|F^+|\}$.  Then there exists a $F_1^+\subset F^+$ such that $|F_1^+|=n$ and $\Var\left[\log P\left(d_i\mid F^+_1,\mathcal{P},\Xi^{\min}\right)\cdot \mathcal{P}\right]$ is approximately minimized for every $d_i\in D$.  
\end{prop}

\begin{proof}
Let the r.v. $\mathcal{Q}_i$ denote $P\left(d_i\mid F^+_1,\mathcal{P},\Xi^{\min}\right)\cdot \mathcal{P}$. And let $\gamma > 1$ denote the expected value of $\exp\left[{c\sum_{j=1}^{|F^+_1|}{\xi^{\min}_{j} \mathbf{1}_{ji}} }\right]$, where the r.v. $\mathbf{1}_{ji}$ models the likelihood of whether $d_i\in\pi(f_j)$. Now we can express $\mathcal{Q}_i$ as $\mathcal{Q}_i = {\gamma\over \gamma \mathcal{P} + 1-\mathcal{P}}$ and reduce $\Var\left[\log\mathcal{Q}_i\right]$ to simple functions of $\mathbf{E}\left[\mathcal{P}\right]$ and $\Var\left[\mathcal{P}\right]$, which are known quantities of the uniform mean (Bates) distribution.
\small
\begin{equation}\label{pe1}
\begin{split}
&\Var\left[\log\mathcal{Q}_i\right] = \Var\left[\log {\gamma \over (\gamma-1)\mathcal{P}+1} \right] = \Var\left[\log \left((\gamma-1)\mathcal{P}+1\right) \right]\approx {\Var\left[(\gamma-1)\mathcal{P}+1 \right]\over \left(\mathbf{E}\left[(\gamma-1)\mathcal{P}+1\right]\right)^2},\text{where}\\
&{\Var\left[(\gamma-1)\mathcal{P}+1 \right]\over \left(\mathbf{E}\left[(\gamma-1)\mathcal{P}+1\right]\right)^2}
=  {(\gamma-1)^2 \Var\left[\mathcal{P}\right]\over \left[(\gamma-1) \mathbf{E}\left[\mathcal{P}\right] +1\right]^2} = {1\over 12n}\left({2{\gamma-1\over 1+p}\over {\gamma-1\over 1+p} + 1}\right)^2 = {1\over 3n}\left({1\over 1+ {1+p\over \gamma -1}}\right)^2.\\
\end{split}
\end{equation}
\normalsize
The ``$\approx$'' in Equation \ref{pe1} is the result of Taylor series expansion on $\log \left((\gamma-1)\mathcal{P}+1\right)$, a common resort to approximate the moments of a ($\log$-)transformed random variable~\citep{vdv}.
Approximately, $\Var\left[\log\mathcal{Q}_i\right] \propto\gamma$. Observe that, for fixed $n$, choosing the $n$ smallest $\xi^{\min}_{j} \mathbf{E}\left[ \mathbf{1}_{ji}\right]$'s will guarantee the smallest $\gamma$. We show the existence of $F_1^+\subset F^+$ by the following construction: consecutively selecting the $f_j^+$'s associated with the $n$ smallest $\xi^{\min}_{j}\mathbf{E}\left[ \mathbf{1}_{ji}\right]$'s. 
\end{proof}

Proposition \ref{p2} states the existence and the construction of $F_1^+\subset F^+$ for each $n$. However, the construction of  $F_1^+$ involves calculating $\gamma$ for each $\mathcal{Q}_i$ and $\xi^{\min}_j$ for all $f_j\in F^+$, which makes the ordering algorithm slower than the actual variational transformation (so is GDO). 

Now we show how to simplify the construction algorithm of $F_1^+$ to FDO without calculating $\gamma$'s  or $\xi_j^{\min}$'s.
For a wide range of practical parameter settings we are interested in (e.g., Figure \ref{fig:gamma} subfigures), we notice that $\gamma$ is empirically $\propto \left(\sum_{j=1}^{|F^+_1|} \xi_j^{\min}\right)^{-1}$. The exact analysis of this claim may be transcendent but $\lim_{\xi_j^{\min}\to\infty}\mathbf{E}\left[|\pi\left(f^+_j\right)|\right]\xi_j^{\min}=0$, suggesting that $\gamma$ eventually approaches minimum when $F^+_1$ is made of $f^+$'s that have the largest $\xi^{\min}$'s. Proposition \ref{p1} shows that $\mathbf{E}\left[|\pi\left(f_j^+\right)|\right]\propto {1\over \xi_j^{\min}}$ for any fixed $p\in[0,+\infty)$, $c\in(0,+\infty)$. As a result, we have $\sum_{j=1}^{|F^+_1|}\mathbf{E}\left[|\pi\left(f_j^+\right)|\right]\propto\gamma$. In other words, compose $F^+_1$ with the $f^+_j$'s that have the smallest $|\pi\left(f_j^+\right)|$ yields the minimal $\gamma$.

\begin{prop}\label{p1}
Fix $p\in[0,+\infty)$, $c\in(0,+\infty)$. Then for any $j\in\left\{j|\pi\left(f^+_j\right)\ne\emptyset\right\}$, its variational parameter $\xi^{\min}_j$ decreases monotonically on $(0,+\infty)$ as $\mathbf{E}\left[|\pi\left(f^+_j\right)|\right]$ increases.
\end{prop}

\begin{proof}
$\xi^{\min}_j$ can be solved from either $\arg\min_{\xi_j} P\left(f_j^+\mid \xi_{j},\pi(f_j)^+\right)$ or $\arg\min_{\xi_j} P\left(f_j^+\mid \xi_{j}\right)$. Since $\xi^{\min}_j=\arg\min_{\xi_j} P\left(f_j^+\mid \xi_{j},\pi(f_j)^+\right)$ can be seen as the special case when $p=0$, our argument below applies to both cases.

For fixed $p,c$, we can solve for $\xi^{\min}_j$ by letting ${\partial\over \partial \xi_j}\log P(f_j^+\mid \xi_j) = 0$. We have $\mathbf{E}\left[|\pi\left(f^+_j\right)|\right] = {1\over c}\log\left(1+{1\over \xi^{\min}_j}\right)\left(p e^ce^{-\xi^{\min}_j} + 1\right)$. Taking derivative of $\mathbf{E}\left[|\pi\left(f^+_j\right)|\right]$ w.r.t. $\xi^{\min}_j$ gives:
\small
\begin{equation}
\begin{split}
{d\mathbf{E}\left[|\pi\left(f^+_j\right)|\right]\over d\xi^{\min}_j} &= - {e^{\xi^{\min}_j}  + pe^c\left[1+\xi^{\min}_j\left(\xi^{\min}_j+1\right)\log(1+{1\over \xi^{\min}_j})\right]\over ce^{\xi^{\min}_j}\xi^{\min}_j(\xi^{\min}_j+1)}<0 \text{, for } \xi^{\min}_j >0.
\end{split}
\end{equation}
\normalsize
\end{proof}

Since the same strategy minimizes $\Var\left[\log\mathcal{Q}_i\right]$ for every $d_i\in D$, it must be the most stable globally as well. 
Therefore, we arrive at an extremely simple variational transformation algorithm: sort $f_j^+\in F^+$ by ascending rank of $\pi\left(f_j^+\right)$ and let that order be the order of variational transformation. We refer to this strategy as finding-degree order (FDO).

\begin{figure}
\centering
\subfigure[]{
  \includegraphics[width=3.3cm]{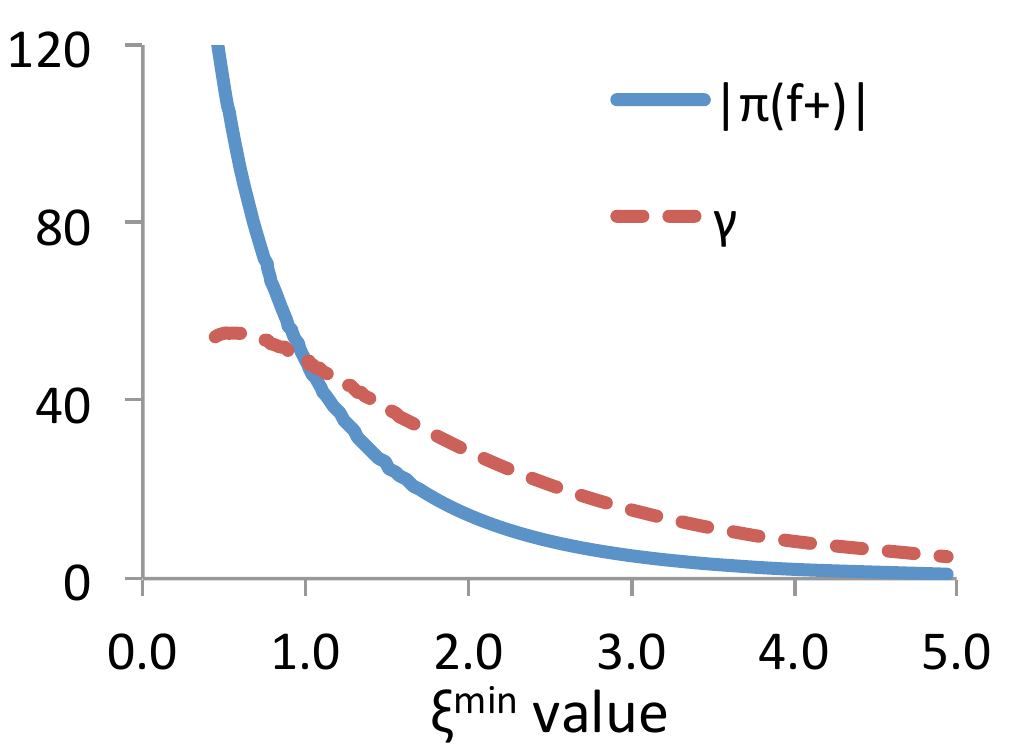}
}\hspace{-0.5mm}%
\subfigure[]{
   \includegraphics[width=3.3cm]{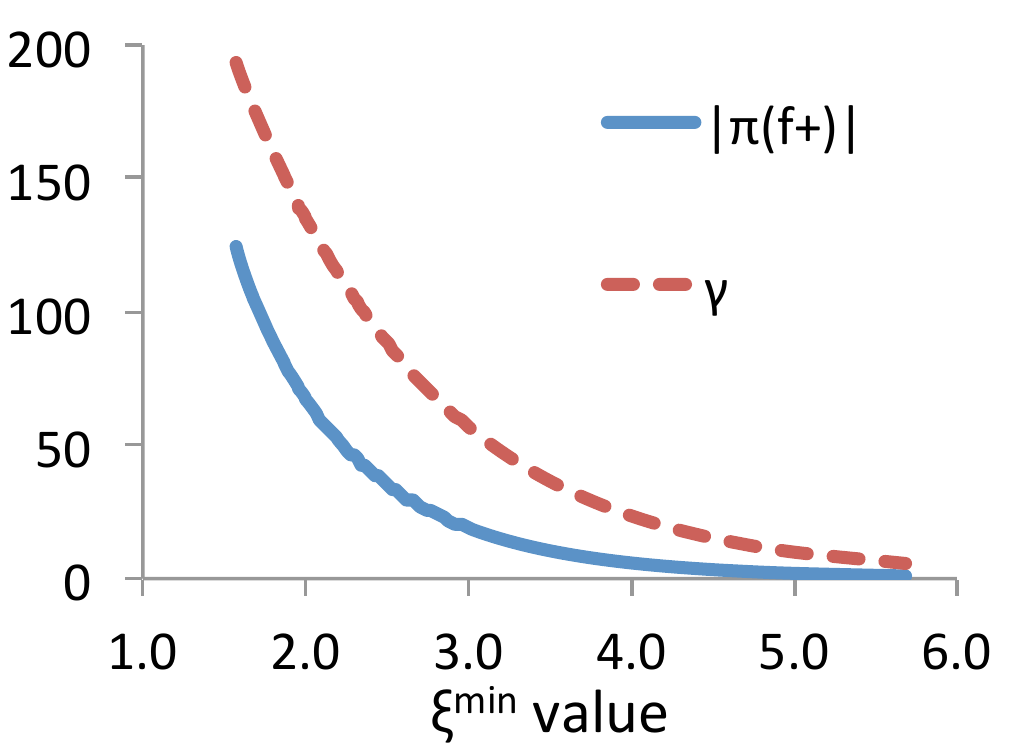}
}\hspace{-0.5mm}%
\subfigure[]{
   \includegraphics[width=3.3cm]{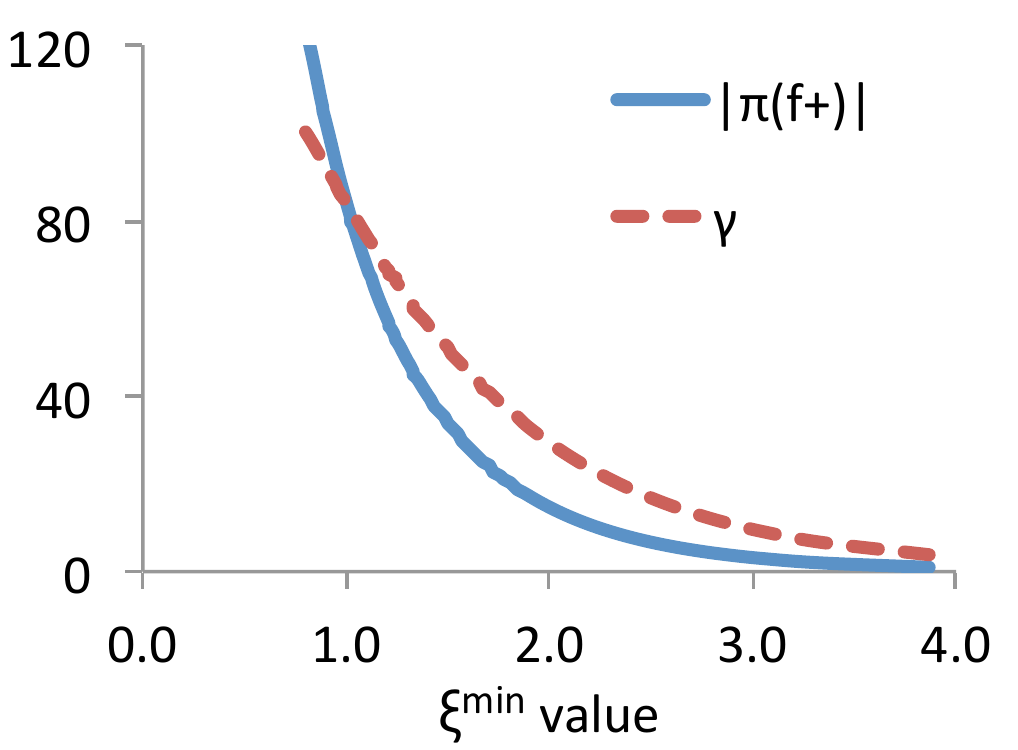}
}\hspace{-0.5mm}%
\subfigure[]{
   \includegraphics[width=3.3cm]{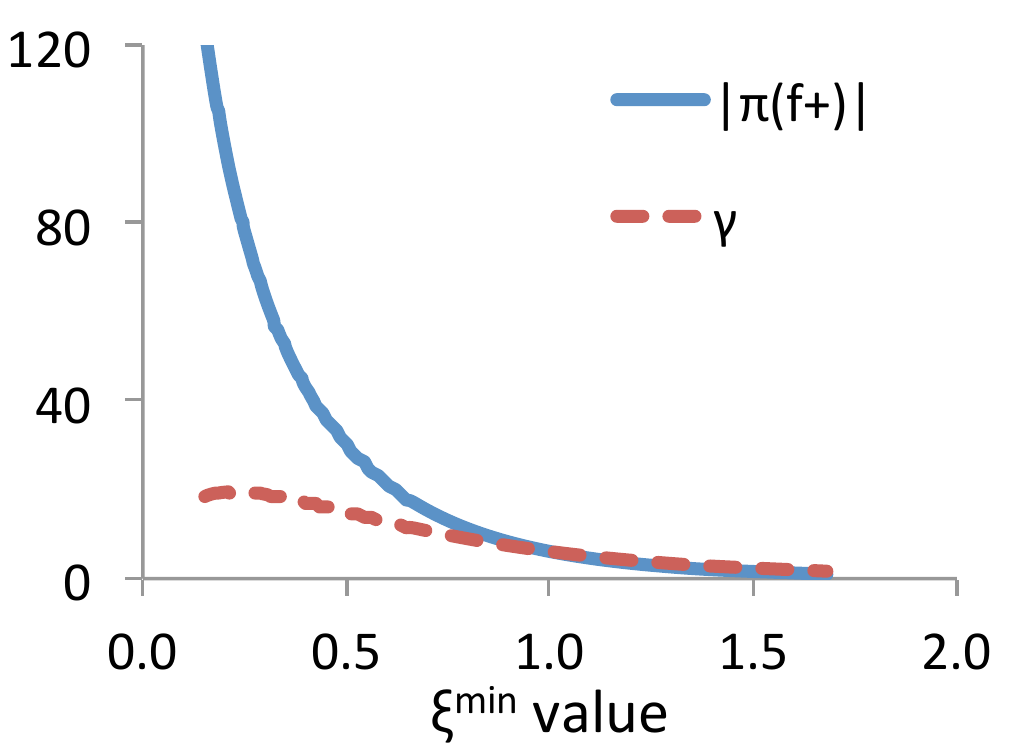}
}
\vspace{-0.5cm}
\caption{{\color{blue}Blue:} $\mathbf{E}\left[|\pi\left(f^+\right)|\right]$ ($y$-axis) vs. $\xi^{\min}$ ($x$-axis) on synthesized experimental data. \\
   {\color{red}Red:}  $\gamma$ ($y$-axis) vs. the  average of $\xi^{\min}$'s correspond to that $\gamma$ ($x$-axis).  \\(a) $p={1-0.01\over0.01}$, $c=-\log \left(1-0.5\right)$. (b) $p={1-0.001\over0.001}$, $c=-\log \left(1-0.6\right)$. (c) $p={1-0.002\over0.002}$, $c=-\log \left(1-0.7\right)$. (d) $p={1-0.005\over0.005}$, $c=-\log \left(1-0.9\right)$.}
\label{fig:gamma}
\end{figure}

\section{Related work}
Exact inference on NOBN is fundamentally intractable~\citep{cooper90}.
Brute force inference on NOBN is $O(|F|\cdot2^{|D|})$ as it calculates $P(F)$ by summing up $P(F\mid D')\cdot P(D')$, where $D'$ can be the combination of the presence or the absence of any subsets of $D$. Junction tree algorithms~\citep{pearl88} can be more efficient in practice at $O(2^{|M|})$, where $|M|$ is the maximal clique size of the moralized network. 

Quickscore~\citep{qs} reduces the temporal complexity to some exponential function of a quantity substantially smaller than $|D|$ or $|M|$ and make the inference practical for common usage. Quickscore~\citep{qs} achieves  $\widetilde{O}\left(|D|\cdot 2^{|F|}\right)$ by exploiting marginal and conditional independence\footnote{the soft-$O$ bound is derived from $O\left(|D|\cdot |F^-| \cdot 2^{|F^+|}\right)$ given in \citep{qs}.}.

\begin{table}[ht!]
\centering
\caption{Overall temporal complexities for exact and variational inferences on NOBN in terms of $|D|$, $|M|$, $|F|$, and $|F'|$ (note that all results are independent of $|S|$). In practical applications like \texttt{QMR-DT}, $|D|=534$, $|M|\approx 151$, and $|F|\approx 43$~\citep{j99,jj99}.}
\label{tab:nobn}
\vspace{-0.2cm}
{\setlength{\extrarowheight}{6pt}%
\begin{tabular}{ c | c | c | c }
Brute force & Junction tree & Quickscore & Variational \\ 
\hline
\small  $O\left(|F|\cdot2^{|D|}\right)$ & \small $O\left(|D|\cdot2^{|M|}\right)$ & \small $\widetilde{O}\left(|D|\cdot 2^{|F|}\right)$ &\small $\widetilde{O}\left(|D|\cdot \left[|F'| +2^{|F-F'|}\right]\right)$ \\
\end{tabular}
}
\end{table}

Various approximate inference methods are proposed in place of Quickscore when processing expensive inference cases in NOBN (particularly \texttt{QMR-DT}).
Variational inference for NOBN developed in \citep{jj99} reduces the cost in computing $P(F)$ by applying variational transformation to a subset of $F'\subset F$. The variational evidence is incorporated as posterior probability when performing quickscore on the remaining findings. The running time is then \small$\widetilde{O}\left(|D|\cdot \left[|F'| +2^{|F-F'|}\right]\right)$\normalsize.

Other general approximation methods that can be applied to NOBN include loopy belief propagation~\citep{loopy}, mean field approximation~\citep{ng00}, and importance sampling based sampling methods~\citep{uai10}.
Some have also considered processing each finding in $F$ sequentially~\citep{pami13}, which is arguably more similar to the style of a realistic patient-to-doctor diagnosis.

\section{Experiments}

We evaluate the proposed inference algorithms on a real-world symptom-disease NOBN called \texttt{F120}.
\texttt{F120} is a QMR-like medical NOBN constructed from multiple reliable medical knowledge sources and is amended by medical experts. Unlike \texttt{QMR-DT}, \texttt{F120} focuses on symptoms and diseases related to maternal and infant care. Due to the anonymous submission, the authors refrain from discussing \texttt{F120}'s details other than listing its vital statistics in Table \ref{tab:data}.

Due to the unavailability of the proprietary \texttt{QMR-DT} network~\citep{uai06}, an anonymized version (\texttt{aQMR}) is available~\citep{uai13}. However, \texttt{aQMR} anonymizes the symptom and disease node names and randomizes \texttt{QMR-DT}'s $P(f^+\mid d^+)$ probabilities. With the medical connotation removed, it is difficult to confidently generate user queries (a user query is a tuple $\langle F^+ , F^- , d_l\rangle$, where $d_l$ is the label disease: the most likely disease given the symptoms according to medical experts). Previous works working with \texttt{aQMR} do not face this issue since they do not require use-cases. For example, \citep{uai13,nips13} focus on recovering the network structure and parameters; \citep{uai10} focuses on the inference time and the \emph{relative} divergence between approximate inference outcome and the exact inference outcome.

We also evaluate the algorithm's scalability on the artificially generated \texttt{S1} that is much larger in scale than \texttt{F120} and \texttt{QMR-DT}. \texttt{S1} has 40,000 hidden disease nodes, which is approximately the total number of diseases in ICD-10 classification. Figure \ref{fig:speed} compares various inference algorithms against the baseline in \citep{jj99} (JJ99+CVX). The proposed variational-first hybridization (VFH) is consistently faster than other methods. Despite having the same variational cost as VFH (shown in Table \ref{tab:variational}), Joint hybridization (JH) is the slowest due to its repeated negative evidence computation of Equation \ref{e11}. JJ99+PPF is significantly faster than JJ99+CVX due to the simplified $\Xi^{\min}$ estimation.

Figure \ref{fig:accuracy} compares the inference accuracies on \texttt{F120}. To simulate the wide-ranged inaccuracy in the disease priors $P(d^+)$'s, we scramble them with samples drawn from the uniform mean (Bates) distribution $\mathcal{P}$ at different ${1\over 1+p}$ values. In total, we test four sets of queries with different kinds of false positive findings. Each query in the 1st set (\texttt{random20}) contains 20\% random $f^+$'s that are not caused by the labeled disease. For the 2nd set (\texttt{chronic20}), the 20\% false $f^+$'s are symptoms caused by some common chronic diseases (e.g., asthma, hypertension). Chronic symptoms are often mentioned inadvertently by patients during doctor's visit and making the diagnosis harder. The 3rd set (\texttt{chronic40}) has the same type of false $f^+$'s as \texttt{chronic20} but the ratio is 40\% of $F^+$. For the 4th set (\texttt{confuse20}), the 20\% false $f^+$'s are symptoms caused by diseases similar to the labeled disease (e.g., influenza and common cold). Such diseases share several symptoms, but often the severity and other key symptoms are decisive in telling them apart. Each of the four sets has 800 queries and each query consists of on average eight $f^+$'s and four $f^-$'s. Shown in Figure \ref{fig:accuracy}, VFH+CVX+FDO performs better than the JJ99+CVX+GDO baseline across the wide range of $P(d^+)$ and even outperforms the exact Quickscore for certain $P(d^+)$ values. VFH+PPF+FDO suffers from its suboptimal (although fast, shown in Figure \ref{fig:speed}) $\xi^{\min}$ estimations. VFH+PPF+FDO is comparable to JJ99+CVX+JJ99 at the lower range of $P(d^+)$ values. Lastly, JH+CVX+FDO has the closest performance portfolio to that of Quickscore and is quite competitive.

\begin{figure}[ht!]
\begin{floatrow}
\capbtabbox{%
\setlength{\extrarowheight}{8pt}%
\begin{tabular}{  l | l | l | l | l }
\textbf{NOBN} & $|D|$ & $|S|$ & \tiny$|\{P(f^+\mid d^+) >0\}|$\normalsize & \textbf{Density}\\ 
     \hline 
\texttt{F120}             & 665 & 1,276 & 10,552 & 1.24\%\\
\texttt{QMR-DT}       & 534 & 4,040 & 40,740 & 1.89\%\\
\texttt{S1}                 &40,000 & 12,000 & 384 million & 80.0\%\\
\end{tabular}
}{%
\captionsetup{position=above}
  \caption{Comparisons of NOBNs on the network size and density (measured as total number of nonzero $P(f^+\mid d^+)$ as a percentage of $|D|\cdot |F|$).}\label{tab:data}%
}\hspace{-1.2cm}
\ffigbox{%
  \includegraphics[width=4.6cm]{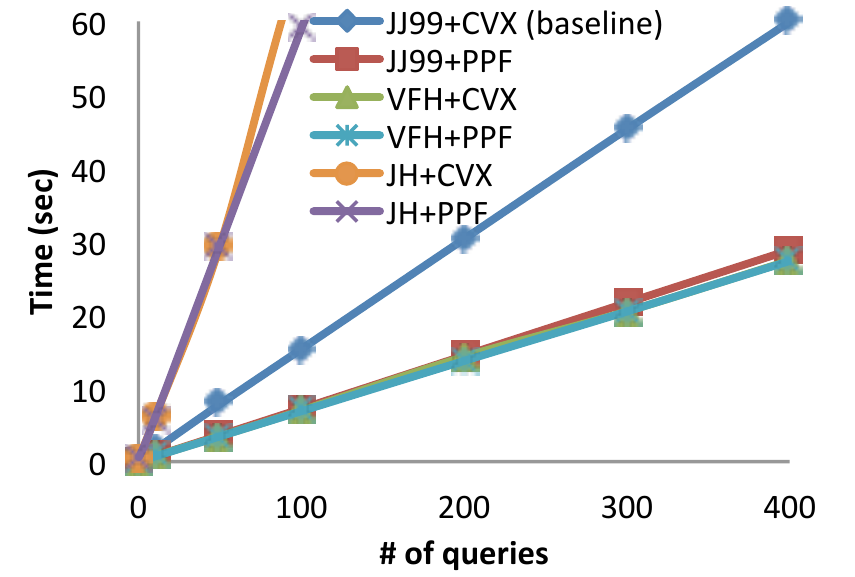}
}{%
\captionsetup{justification=centering}
  \caption{Runtime comparisons \\ of different algorithms \\ on the \texttt{S1} network.}\label{fig:speed}%
}
\end{floatrow}
\end{figure}

\begin{figure}[ht!]
\centering
\subfigure[]{
  \includegraphics[width=3.65cm]{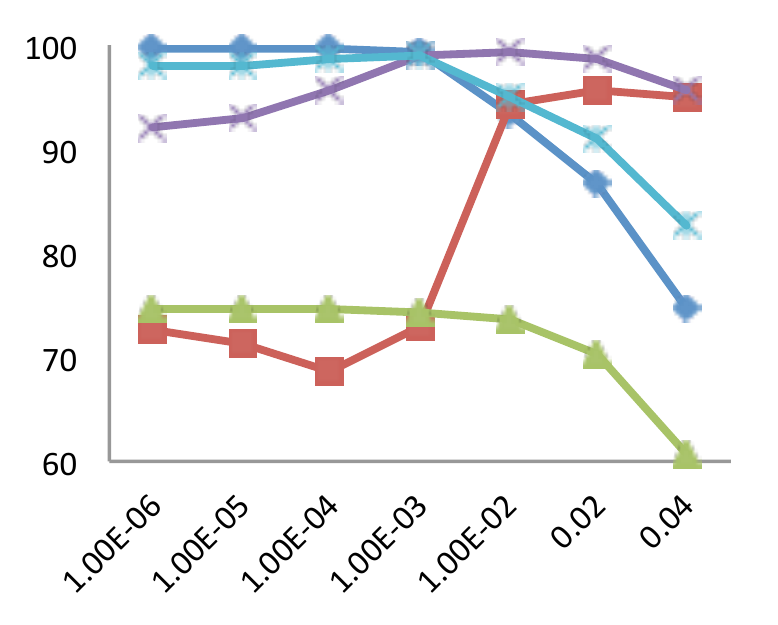}
}\hspace{-0.5cm}%
\subfigure[]{
   \includegraphics[width=3.65cm]{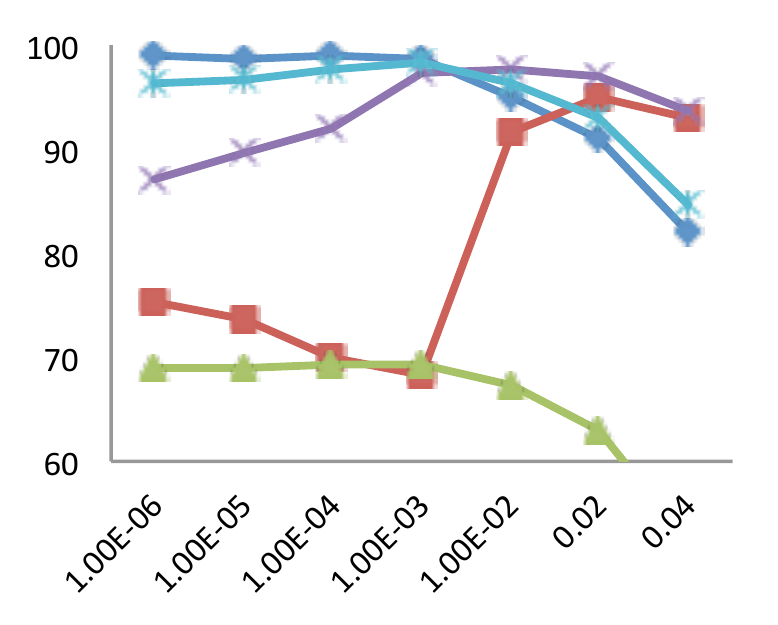}
}\hspace{-0.5cm}%
\subfigure[]{
   \includegraphics[width=3.65cm]{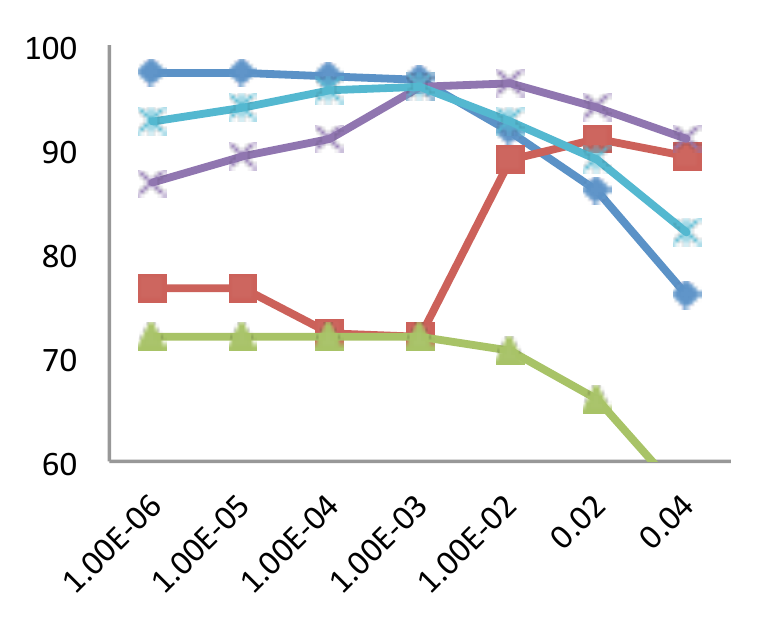}
}\hspace{-0.5cm}%
\subfigure[]{
   \includegraphics[width=3.65cm]{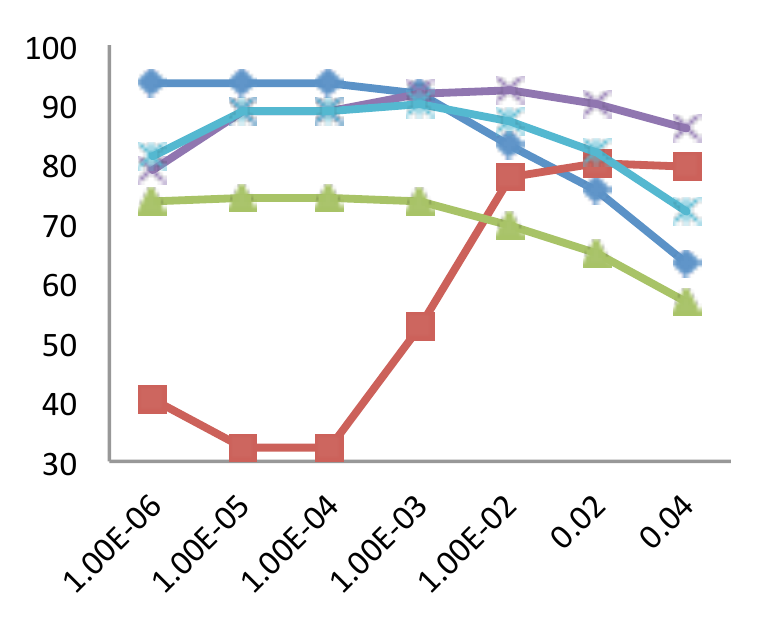}
}\\[-0.4cm]
\subfigure[]{
  \includegraphics[width=3.65cm]{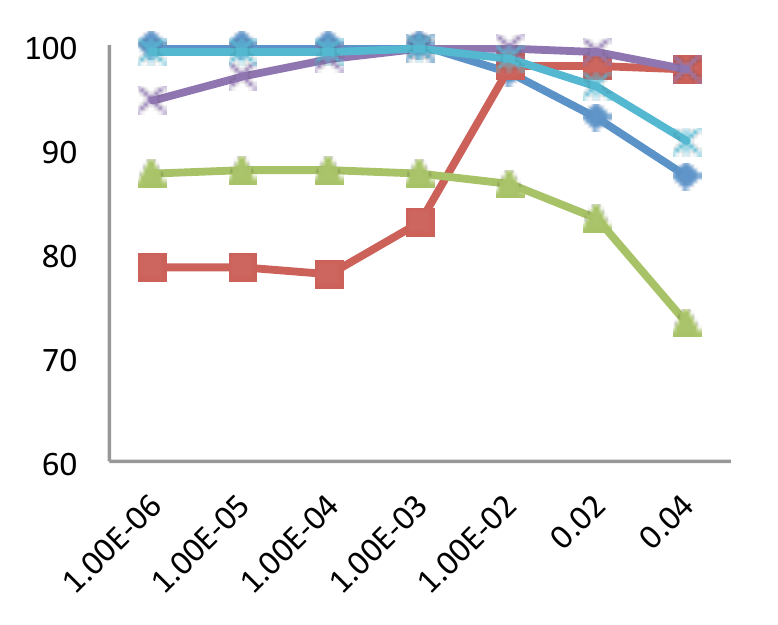}
}\hspace{-0.5cm}%
\subfigure[]{
   \includegraphics[width=3.65cm]{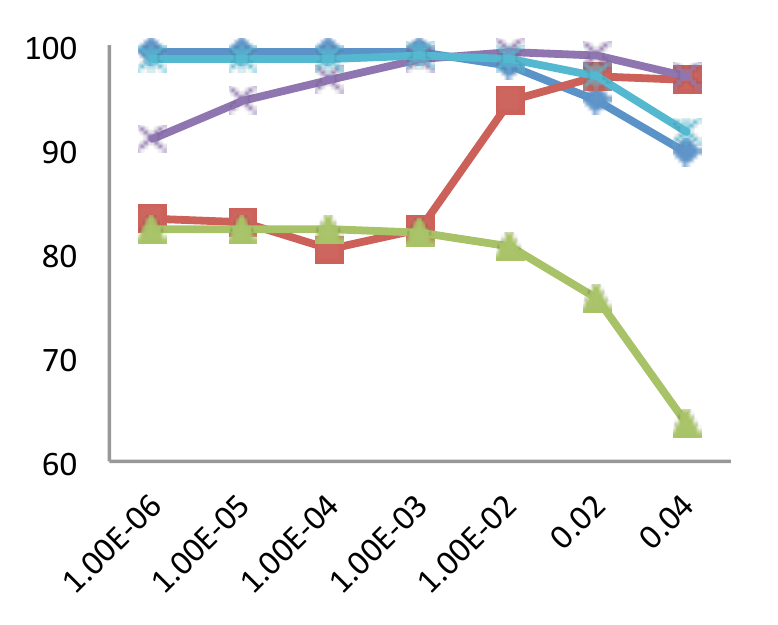}
}\hspace{-0.5cm}%
\subfigure[]{
   \includegraphics[width=3.65cm]{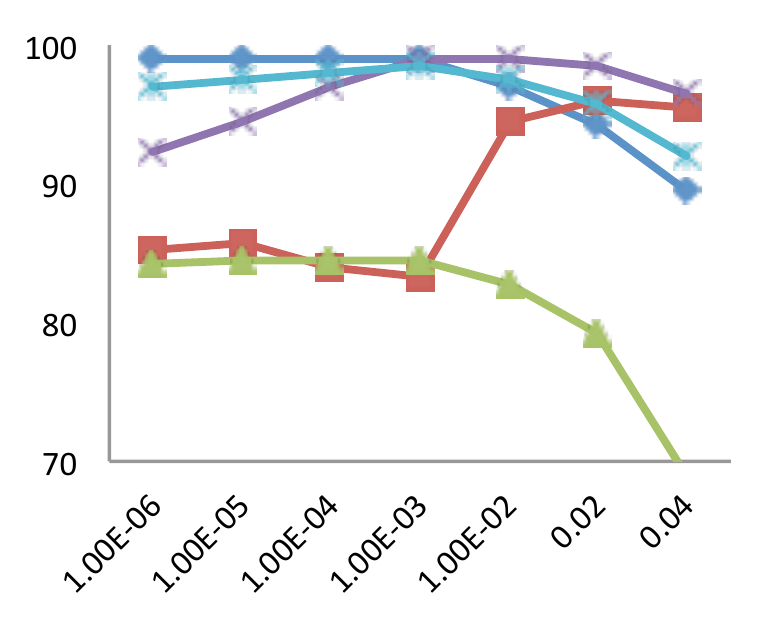}
}\hspace{-0.5cm}%
\subfigure[]{
   \includegraphics[width=3.65cm]{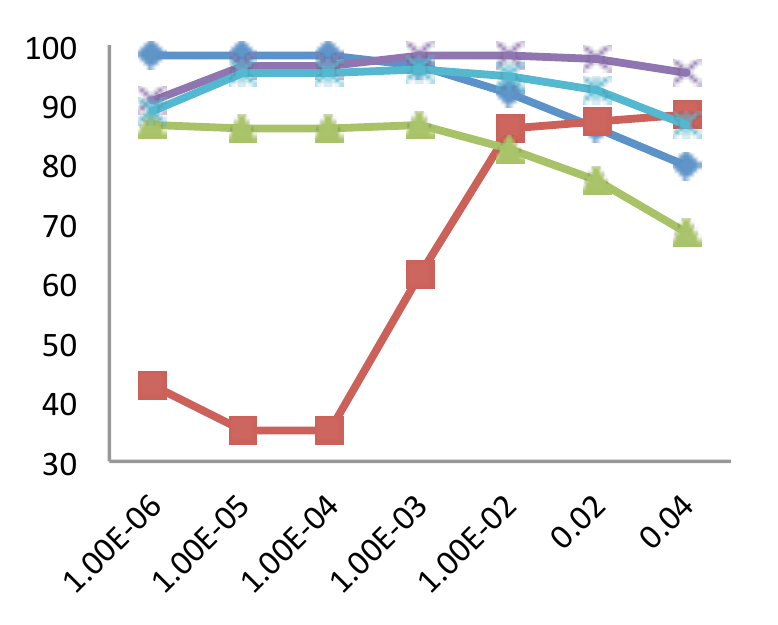}
}\\[-0.1cm]
   \includegraphics[width=12cm]{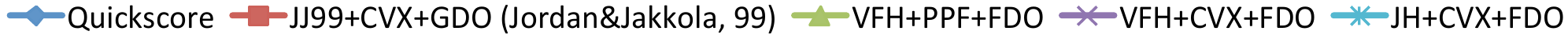}
   \vspace{-0.3cm}
\caption{ Accuracy comparisons on \texttt{F120}. $x$-axis is the mean $P(d^+)$ value (i.e., ${1\over 1+p}$); $y$-axis is top-1/top-3 accuracy. All configurations (except Quickscore) transforms 2 findings variationally.  \newline (a, e) \texttt{random20}. (b, f)  \texttt{chronic20}. (c, g) \texttt{chronic40}. (d, h) \texttt{confuse20}. \newline (a, b, c, d) measure top-1 accuracies. (e, f, g, h) measure the corresponding top-3 accuracies.}
\label{fig:accuracy}
\end{figure}

\section{Conclusions and future work}
In this work, we study the important problem of approximate inference on noisy-or Bayesian networks (specifically, their medical applications). We introduce novel algorithms for variational hybridization and variational transformation. The proposed algorithms greatly immunize the current variational inference algorithms against the inaccuracies in widely-ranged hidden prior probabilities, a common issue that arises in modern medical applications of Bayesian networks. In the future, we plan to investigate the applicability of the proposed algorithms to more general Bayesian networks.

\small
\bibliography{nips_2016}

\begin{thebibliography}{21}
\providecommand{\natexlab}[1]{#1}
\providecommand{\url}[1]{\texttt{#1}}
\expandafter\ifx\csname urlstyle\endcsname\relax
  \providecommand{\doi}[1]{doi: #1}\else
  \providecommand{\doi}{doi: \begingroup \urlstyle{rm}\Url}\fi

\bibitem[Amodei et~al.(2015)Amodei, Anubhai, Battenberg, Case, Casper,
  Catanzaro, Chen, Chrzanowski, Coates, Diamos, Elsen, Engel, Fan, Fougner,
  Han, Hannun, Jun, LeGresley, Lin, Narang, Ng, Ozair, Prenger, Raiman,
  Satheesh, Seetapun, Sengupta, Wang, Wang, Wang, Xiao, Yogatama, Zhan, and
  Zhu]{deepspeech}
D.~Amodei, R.~Anubhai, E.~Battenberg, C.~Case, J.~Casper, B.~C. Catanzaro,
  J.~Chen, M.~Chrzanowski, A.~Coates, G.~Diamos, E.~Elsen, J.~Engel, L.~Fan,
  C.~Fougner, T.~Han, A.~Y. Hannun, B.~Jun, P.~LeGresley, L.~Lin, S.~Narang,
  A.~Y. Ng, S.~Ozair, R.~Prenger, J.~Raiman, S.~Satheesh, D.~Seetapun,
  S.~Sengupta, Y.~Wang, Z.~Wang, C.~Wang, B.~Xiao, D.~Yogatama, J.~Zhan, and
  Z.~Zhu.
\newblock Deep speech 2: End-to-end speech recognition in english and mandarin.
\newblock \emph{CoRR}, abs/1512.02595, 2015.
\newblock URL \url{http://arxiv.org/abs/1512.02595}.

\bibitem[Bellala et~al.(2013)Bellala, Stanley, Bhavnani, and Scott]{pami13}
G.~Bellala, J.~Stanley, S.~K. Bhavnani, and C.~Scott.
\newblock A rank-based approach to active diagnosis.
\newblock \emph{IEEE Transactions on Pattern Analysis and Machine
  Intelligence}, 35\penalty0 (9):\penalty0 2078--2090, Sept 2013.
\newblock ISSN 0162-8828.
\newblock \doi{10.1109/TPAMI.2013.30}.

\bibitem[Cheng et~al.(2002)Cheng, Greiner, Kelly, Bell, and Liu]{ai02}
J.~Cheng, R.~Greiner, J.~Kelly, D.~Bell, and W.~Liu.
\newblock Learning bayesian networks from data: An information-theory based
  approach.
\newblock \emph{Artificial Intelligence}, 137\penalty0 (1 - 2):\penalty0 43 --
  90, 2002.
\newblock ISSN 0004-3702.
\newblock \doi{http://dx.doi.org/10.1016/S0004-3702(02)00191-1}.
\newblock URL
  \url{http://www.sciencedirect.com/science/article/pii/S0004370202001911}.

\bibitem[Cooper(1990)]{cooper90}
G.~F. Cooper.
\newblock The computational complexity of probabilistic inference using
  bayesian belief networks (research note).
\newblock \emph{Artif. Intell.}, 42\penalty0 (2-3):\penalty0 393--405, Mar.
  1990.
\newblock ISSN 0004-3702.
\newblock \doi{10.1016/0004-3702(90)90060-D}.
\newblock URL \url{http://dx.doi.org/10.1016/0004-3702(90)90060-D}.

\bibitem[Dawson and Dellavalle(2013)]{acne}
A.~L. Dawson and R.~P. Dellavalle.
\newblock Acne vulgaris.
\newblock \emph{BMJ}, 346, 2013.
\newblock \doi{10.1136/bmj.f2634}.
\newblock URL \url{http://www.bmj.com/content/346/bmj.f2634}.

\bibitem[Gogate and Domingos(2010)]{uai10}
V.~Gogate and P.~M. Domingos.
\newblock Formula-based probabilistic inference.
\newblock In P.~Grunwald and P.~Spirtes, editors, \emph{UAI}, pages 210--219.
  AUAI Press, 2010.

\bibitem[Halpern and Sontag(2013)]{uai13}
Y.~Halpern and D.~Sontag.
\newblock Unsupervised learning of noisy-or bayesian networks.
\newblock In \emph{Proceedings of the 29th Conference, UAI}, pages 272--281,
  2013.

\bibitem[Heckerman(1990)]{qs}
D.~Heckerman.
\newblock A tractable inference algorithm for diagnosing multiple diseases.
\newblock In \emph{UAI}, pages 163--172, 1990.

\bibitem[Jaakkola and Jordan(1999)]{jj99}
T.~S. Jaakkola and M.~I. Jordan.
\newblock Variational probabilistic inference and the qmr-dt network.
\newblock \emph{Journal of artificial intelligence research}, pages 291--322,
  1999.

\bibitem[Jernite et~al.(2013)Jernite, Halpern, and Sontag]{nips13}
Y.~Jernite, Y.~Halpern, and D.~Sontag.
\newblock Discovering hidden variables in noisy-or networks using quartet
  tests.
\newblock In \emph{Advances in Neural Information Processing Systems 26}, pages
  2355--2363, 2013.

\bibitem[Jordan et~al.(1999)Jordan, Ghahramani, Jaakkola, and Saul]{j99}
M.~I. Jordan, Z.~Ghahramani, T.~S. Jaakkola, and L.~K. Saul.
\newblock An introduction to variational methods for graphical models.
\newblock \emph{Machine learning}, 37\penalty0 (2):\penalty0 183--233, 1999.

\bibitem[Liao and Ji(2009)]{pr09}
W.~Liao and Q.~Ji.
\newblock Learning bayesian network parameters under incomplete data with
  domain knowledge.
\newblock \emph{Pattern Recognition}, 42\penalty0 (11):\penalty0 3046 -- 3056,
  2009.
\newblock ISSN 0031-3203.
\newblock \doi{http://dx.doi.org/10.1016/j.patcog.2009.04.006}.
\newblock URL
  \url{http://www.sciencedirect.com/science/article/pii/S0031320309001472}.

\bibitem[Luong et~al.(2015)Luong, Sutskever, Le, Vinyals, and Zaremba]{nmt}
T.~Luong, I.~Sutskever, Q.~V. Le, O.~Vinyals, and W.~Zaremba.
\newblock Addressing the rare word problem in neural machine translation.
\newblock In \emph{ACL}, 2015.
\newblock URL \url{http://arxiv.org/pdf/1410.8206v4.pdf}.

\bibitem[Mansinghka et~al.(2006)Mansinghka, Kemp, Griffiths, and
  Tenenbaum]{uai06}
V.~K. Mansinghka, C.~Kemp, T.~L. Griffiths, and J.~B. Tenenbaum.
\newblock Structured priors for structure learning.
\newblock In \emph{Proceedings of the 22nd Conference, UAI}, 2006.

\bibitem[Middleton et~al.(1991)Middleton, Shwe, Heckerman, Henrion, Horvitz,
  Lehmann, and Cooper]{qmr}
B.~Middleton, M.~Shwe, D.~Heckerman, M.~Henrion, E.~Horvitz, H.~Lehmann, and
  G.~Cooper.
\newblock Probabilistic diagnosis using a reformulation of the internist-1/qmr
  knowledge base.
\newblock \emph{Methods of information in medicine}, 30:\penalty0 241--255,
  1991.

\bibitem[Murphy(2002)]{bnt}
K.~Murphy.
\newblock {Bayes Net Toolbox}.
\newblock \url{https://github.com/bayesnet/bnt}, 2002.
\newblock [Online Lecture notes; accessed Jan-2016].

\bibitem[Murphy et~al.(1999)Murphy, Weiss, and Jordan]{loopy}
K.~P. Murphy, Y.~Weiss, and M.~I. Jordan.
\newblock Loopy belief propagation for approximate inference: An empirical
  study.
\newblock In \emph{UAI}, pages 467--475, 1999.

\bibitem[Ng and Jordan(2000)]{ng00}
A.~Y. Ng and M.~I. Jordan.
\newblock Approximate inference a lgorithms for two-layer bayesian networks.
\newblock In \emph{Advances in Neural Information Processing Systems}, pages
  533--539, 2000.

\bibitem[Pearl(1988)]{pearl88}
J.~Pearl.
\newblock \emph{Probabilistic Reasoning in Intelligent Systems: Networks of
  Plausible Inference}.
\newblock Morgan Kaufmann Publishers Inc., San Francisco, CA, USA, 1988.
\newblock ISBN 1558604790.

\bibitem[Riggelsen(2006)]{ijar06b}
C.~Riggelsen.
\newblock Learning parameters of bayesian networks from incomplete data via
  importance sampling.
\newblock \emph{International Journal of Approximate Reasoning}, 42\penalty0 (1
  - 2):\penalty0 69 -- 83, 2006.
\newblock ISSN 0888-613X.
\newblock \doi{http://dx.doi.org/10.1016/j.ijar.2005.10.005}.
\newblock URL
  \url{http://www.sciencedirect.com/science/article/pii/S0888613X05000654}.

\bibitem[van~der Vaart(1998)]{vdv}
A.~W. van~der Vaart.
\newblock \emph{Asymptotic statistics}.
\newblock Cambridge Series in Statistical and Probabilistic Mathematics.
  Cambridge University Press, 1998.
\newblock ISBN 0-521-49603-9.

\end{thebibliography}
\end{document}